\documentclass[10pt,twocolumn,letterpaper]{article}

\usepackage[dvipsnames]{xcolor}

\usepackage{iccv}

\usepackage{soul}
\usepackage{times} 
\usepackage{epsfig}
\usepackage{graphicx}
\usepackage{amsmath}
\usepackage{amssymb}
\usepackage{enumitem}
\setitemize{noitemsep,topsep=0pt,parsep=0pt,partopsep=0pt}

\usepackage{wrapfig}

\usepackage{stfloats}

\usepackage{amsthm}

\usepackage{booktabs}
\usepackage{multirow}
\usepackage[tight,footnotesize,sf,SF]{subfigure}

\usepackage[noend,linesnumbered]{algorithm2e}

\usepackage{capt-of,etoolbox} %

\usepackage{lipsum}
 
\newcommand{\R}{\mathbb{R}}
\newcommand{\N}{\mathbb{N}}
\DeclareMathOperator*{\argmin}{arg\,min}

\newtheorem{theorem}{Theorem}
\newtheorem{lem}[theorem]{Lemma}
\newtheorem{prop}[theorem]{Proposition}
\newtheorem{cor}[theorem]{Corollary}
\newtheorem{definition}[theorem]{Definition}

\newtheorem{remark}[theorem]{Remark}

\newcommand{\perm}{\mathbb{P}}

\newcommand{\VEC}{\operatorname{vec}}

\newcommand{\trace}{\operatorname{tr}}
\newcommand{\diag}{\operatorname{diag}}
\newcommand{\matI}{\mathbf{I}}
\newcommand{\onevec}{\mathbf{1}}

\renewcommand{\paragraph}{\textbf}

\makeatletter
\@tfor\next:=abcdefghijklmnopqrstuvwxyz\do{%
  \def\command@factory#1{%
    \expandafter\def\csname vec#1\endcsname{\mathbf{#1}}
  }
 \expandafter\command@factory\next
}

\@tfor\next:=ABCDEFGHIJKLMNOPQRSTUVWXYZ\do{%
  \def\command@factory#1{%
    \expandafter\def\csname mat#1\endcsname{\mathbf{#1}}
  }
 \expandafter\command@factory\next
}

\@tfor\next:=ABCDEFGHIJKLMNOPQRSTUVWXYZ\do{%
  \def\command@factory#1{%
    \expandafter\def\csname set#1\endcsname{\mathcal{#1}}
  }
 \expandafter\command@factory\next
}

\def\greekvectors#1{%
 \@for\next:=#1\do{%
    \def\X##1;{%
     \expandafter\def\csname mat##1\endcsname{\boldsymbol{\csname##1\endcsname}}
     }
   \expandafter\X\next;
  }
}

\greekvectors{alpha,beta,iota,gamma,lambda,nu,eta,Gamma,varsigma,Phi,Theta}

\makeatother

\usepackage[pagebackref=true,breaklinks=true,letterpaper=true,colorlinks,bookmarks=false]{hyperref}

\def\iccvPaperID{****} %

\ificcvfinal\fi

\newif\ifarxiv

\arxivtrue

\ifarxiv
\iccvfinalcopy %
\fi

\begin{document}

\title{Higher-order Projected Power Iterations for Scalable Multi-Matching}

\author{Florian Bernard\textsuperscript{1,2}$\qquad$
Johan Thunberg\textsuperscript{3}$\qquad$ 
Paul Swoboda\textsuperscript{1,2}$\quad$
Christian Theobalt\textsuperscript{1,2}
\\
$\newline$
\\
\textsuperscript{1}MPI Informatics$\qquad$
\textsuperscript{2}Saarland Informatics Campus$\qquad$
\textsuperscript{3}Halmstad University
}

\makeatletter
\let\@oldmaketitle\@maketitle%
\renewcommand{\@maketitle}{\@oldmaketitle%
  \myfigure{}\bigskip}%
\makeatother

\newcommand\myfigure{%
  \centerline{
  \includegraphics[scale=.39]{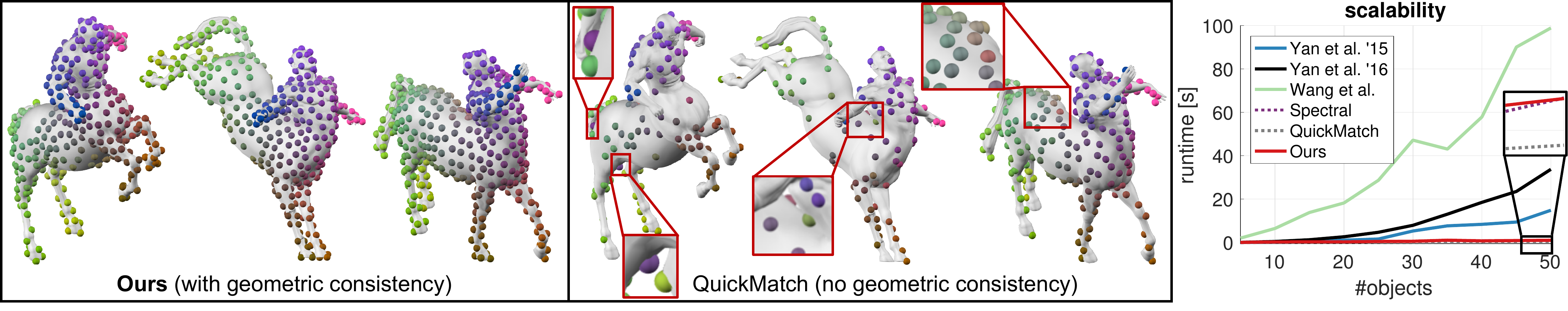}
  }
  \vspace{-1.5mm}
    \captionof{figure}{We propose a novel multi-matching method that is both \emph{scalable} and able to account for \emph{geometric consistency}. Left:~the \emph{centaurs} (in different poses) are consistently matched by our approach, as indicated by the coloured dots. Centre:~methods that ignore geometric relations (\textsc{QuickMatch}~\cite{Tron:kUBrCZhd}, \textsc{Spectral}~\cite{Pachauri:2013wx}) lead to wrong matchings, as evidenced by mismatching colours in the magnifications. Right:~existing methods that account for geometric consistency (Yan et~al.~'15~\cite{Yan:2015vc}, Yan et~al.~'16~\cite{Yan:2016vf}, Wang et al.~\cite{Wang:2017ub}) do not scale to large problems (shown are runtimes when matching $5$--$50$ objects, each having $20$ points, so that a total of $100$--$1000$ points are matched).
    }
    \label{fig:teaser}
  }

\maketitle
\begin{abstract}
The matching of multiple objects (e.g. shapes or images) is a fundamental problem in vision and graphics. In order to robustly handle ambiguities, noise and repetitive patterns in challenging real-world settings, it is essential to take geometric consistency between points into account. Computationally, the multi-matching problem is difficult. It can be phrased as simultaneously solving multiple (NP-hard) quadratic assignment problems (QAPs) that are coupled via cycle-consistency constraints. The main limitations of existing multi-matching methods are that they either ignore geometric consistency and thus have limited robustness, or they are restricted to small-scale problems due to their (relatively) high computational cost. We address these shortcomings by introducing a Higher-order Projected Power Iteration method, which is (i) efficient and scales to tens of thousands of points, (ii) straightforward to implement, (iii) able to incorporate geometric consistency, (iv) guarantees cycle-consistent multi-matchings, and (iv) comes with theoretical convergence guarantees. Experimentally we show that our approach is superior to existing methods.
\end{abstract}

\section{Introduction}
Establishing correspondences is a fundamental problem in vision and graphics that is relevant in a wide range of applications, including reconstruction, tracking, recognition, or matching. 
The goal of correspondence problems is to identify points in objects (e.g.~images, meshes, or graphs) that are semantically similar.
 While matching points independently of their neighbourhood context
is computationally tractable (e.g.~via the \emph{linear assignment problem} (LAP) \cite{Burkard:2009hp}), %
such approaches are limited to simple cases without ambiguities or repetitive patterns. In order to resolve ambiguities and avoid mismatches in challenging real-world scenarios, it is crucial to additionally incorporate the geometric context of the points, so that spatial distances between pairs of points are (approximately) preserved by the matching. To this end, higher-order information is commonly integrated into the matching problem formulation, e.g.~via the NP-hard \emph{quadratic assignment problem} (QAP) \cite{Lawler:1963wn}.

\emph{Multi-matching},~i.e.~finding matchings between \emph{more than two objects} (e.g.~an image sequence, a multi-view scene, or a shape collection) plays an important role in various applications, such as video-based tracking, multi-view reconstruction (e.g.~for AR/VR content generation) or shape modelling (e.g.~for statistical shape models in biomedicine \cite{Heimann:2009kv}). Computationally, finding valid matchings between more than two objects simultaneously is more difficult compared to matching only a pair of objects. This is because one additionally needs to account for \emph{cycle-consistency}, which means that compositions of matchings over cycles must be  the identity matching---even when ignoring higher-order terms, an analogous multi-matching variant of the \emph{linear assignment problem} that accounts for cycle-consistency results in a (non-convex) quadratic optimisation problem over binary variables, which structurally resembles the NP-hard quadratic assignment problem. 
When additionally considering higher-order terms in order to account for geometric relations between the points, multi-matching problems become even more difficult. 
For example, a multi-matching version of the quadratic assignment problem either results in a fourth-order polynomial objective function, or in a quadratic objective with additional (non-convex) quadratic \emph{cycle-consistency constraints}, both of which are to be optimised over binary variables.

Practical approaches for solving multi-matching problems can be put into two categories: (i) methods that jointly optimise for multi-matchings between all objects (e.g.~\cite{Huang:2013uk,Yan:2016vf,Shi:2016tj,Tron:kUBrCZhd,bernard:2018}) and (ii) approaches that first establish matchings between points in each pair of objects independently, 
and then improve those matchings via a post-processing procedure referred to as 
\emph{permutation synchronisation} \cite{Pachauri:2013wx,Chen:2014uo,zhou2015multi,Shen:2016wx,Maset:YO8y6VRb}. Approaches that jointly optimise for multi-matchings either ignore geometric relations between the points \cite{Tron:kUBrCZhd}, or are prohibitively expensive and thus only applicable to small problems (cf.~Fig.~\ref{fig:teaser}).
In contrast, while synchronisation-based approaches are generally more scalable (e.g.~synchronisation problems with a total number of points in the order of $10$k-$100$k can be solved), 
they completely ignore geometric relations (higher-order terms) during the synchronisation, and thus achieve limited robustness in ambiguous settings (cf.~Fig.~\ref{fig:teaser}).

The aim of this work is to provide a scalable solution for multi-matching that addresses the mentioned short-comings of previous approaches. Our main contributions are:
\begin{itemize}
	\item We propose a method that \emph{jointly optimises for multi-matchings} which is \emph{efficient} and thus applicable to \emph{large-scale} multi-matching problems.
	\item Our method is guaranteed to produce \emph{cycle-consistent multi-matchings}, while at the same time considering \emph{geometric consistency} between the points.
	\item Our \emph{Higher-order Projected Power Iteration (HiPPI)} method that has \emph{theoretical convergence guarantees} and can be implemented in few lines of code.
	\item We empirically demonstrate that our method achieves beyond state-of-the-art results on various challenging problems, including large-scale multi-image matching and multi-shape matching.
\end{itemize}

\section{Background \& Related Work}
In this section we review the most relevant works in the literature, while at the same time providing a summary of the necessary background.

\paragraph{Pairwise Matching:} The linear assignment problem (LAP) \cite{Burkard:2009hp} can be phrased as 
\begin{align}\label{eq:lap}
	\min_{X \in \perm} ~\langle A, X \rangle\,,
\end{align}
where $A$ is a given matrix that encodes (linear) matching costs between two given objects, $X \in \perm$ is a permutation matrix that encodes the matching between these objects, and $\langle \cdot,\cdot \rangle$ denotes the Frobenius inner product. The LAP can be solved in polynomial time, e.g. via the Kuhn-Munkres/Hungarian method~\cite{Munkres:1957ju} or the (empirically) more efficient Auction algorithm~\cite{Bertsekas:1998vt}. The quadratic assignment problem (QAP)~\cite{Lawler:1963wn}, which reads
\begin{align}\label{eq:qap}
		\min_{X \in \perm} ~ \VEC(X)^T W \VEC(X) \,,
\end{align}
additionally incorporates pairwise matching costs between two objects that are encoded by the matrix $W$. The QAP is a (strict) generalisation of the LAP, which can be seen by defining $W = \diag(\VEC(A))$ and observing that $X \in \perm$ implies $X = X \odot X$. Here, $\VEC(A)$ stacks the columns of $A$ into a column vector, $\diag(x)$ creates a diagonal matrix that has the vector $x$ on its diagonal, and $\odot$ is the Hadamard product. However, in general the QAP is known to be NP-hard~\cite{Pardalos:1993uo}. The QAP is a popular formalism for \emph{graph matching} problems, where the first-order terms (on the diagonal of $W$) account for node matching costs, and the second-order terms (on the off-diagonal of $W$) account for edge matching costs. Existing methods that tackle the QAP/graph matching include spectral relaxations~\cite{Leordeanu:2005ur,Cour:2006un}, linear relaxations \cite{Torresani:2013gj,swoboda2017b}, convex relaxations~\cite{Zhao:1998wc,Schellewald:2005up,Olsson:2007wx,Fogel:2013wt,Aflalo:2015hd,kezurer2015,Dym:2017ue,bernard:2018}, path-following methods~\cite{Zaslavskiy:2009fq,Zhou:2016ty,Jianga:vg}, kernel density estimation~\cite{Vestner:2017tj}, branch-and-bound methods~\cite{Bazaraa:1979fh} and many more, as described in the survey papers~\cite{Pardalos:1993uo,Loiola:2ua4FrR7}. Also, tensor-based approaches for higher-order graph matching have been considered \cite{duchenne2011tensor,Nguyen:2015vq}.

While the requirement $X \in \perm$ implies \emph{bijective} matchings, in the case of matching only two objects, the formulations \eqref{eq:lap} and \eqref{eq:qap} are general in the sense that they also apply to \emph{partial} matchings, which can be achieved by incorporating \emph{dummy points} with suitable costs. However, due to ambiguities with multi-matchings of dummy points, this is not easily possible when considering more than two objects. %

\paragraph{Multi-matching:} 
In contrast to the work \cite{Williams:1997vj}, where cycle-consistency has been modelled as soft-constraint within a Bayesian framework for multi-graph matching, in \cite{Yan:2013ve,Yan:2014wi} the authors have addressed multi-graph matching in terms of simultaneously solving pairwise graph matching under (hard) cycle-consistency constraints. In \cite{Yan:2015vc}, the authors have generalised \emph{factorised graph matching} \cite{Zhou:2016ty} from matching a pair of graphs to multi-graph matching. Another approach that tackles multi-graph matching is based on a low-rank and sparse matrix decomposition \cite{Yan:2015wr}. In \cite{Yan:2016vf}, a composition-based approach with a cycle-consistency regulariser is employed.
In \cite{kezurer2015}, the authors propose a semidefinite programming (SDP) relaxation for multi-graph matching by (i) relaxing cycle-consistency via a semidefinite constraint, and (ii) lifting the $n{\times} n$ permutation matrices to $n^2 {\times} n^2$-dimensional matrices. In order to reduce computational costs due to the lifting of the permutation matrices, the authors in \cite{bernard:2018} propose a lifting-free SDP relaxation for multi-graph matching. In \cite{park2018consistent}, the authors propose a random walk technique for multi-layered multi-graph matching. While there is a wide range of algorithmic approaches for multi-graph matching, the aforementioned approaches have in common that they are computationally expensive and are only applicable to small-scale problems, where the total number of points does not significantly exceed a thousand (e.g.~$20$ graphs with $50$ nodes each). In contrast, our method scales much better and handles multi-matching problems with more than $20$k points.

In \cite{cosmo2017consistent}, the authors use a two-stage approach with a sparsity-inducing $\ell_1$-formulation for multi-shape matching. While the effect of this approach is that only few multi-matchings are found, our approach obtains significantly more multi-matchings, as we will demonstrate later.

Rather than modelling higher-order relations between points, the recent approach \cite{Wang:2017ub} accounts for geometric consistency in 2D multi-image matching problems by imposing a low rank of the (stacked) 2D image coordinates of the feature points. On the one hand, this is based on the (over-simplified) assumption that the 2D images depict a 3D scene under \emph{orthographic} projections, and on the other hand such an extrinsic approach is not directly applicable to distances on non-Euclidean manifolds (e.g.~multi-shape matching with geodesic distances). 
In contrast, our approach is intrinsic due to the use of pairwise adjacency matrices, and thus can handle \emph{general pairwise} information independent of the structure of the ambient space. 

\paragraph{Synchronisation methods:}
Given pairwise matchings between pairs of objects in a collection, \emph{permutation synchronisation}  methods have the purpose to 
achieve \emph{cycle-consistency} in the set of pairwise matchings. For  
\begin{align}\label{eq:partPerm}
    \perm_{pq} = \{X \in \{0,1\}^{p \times q}~:~X\onevec_q \leq \onevec_p, \onevec_p^T X \leq \onevec_q^T\}
  \end{align}
  being the set of of $(p {\times} q)$-dimensional partial permutation matrices, we define cycle-consistency as follows:
\begin{definition} \label{def:cycleCons} (Cycle-consistency)\footnote{Note that matrix inequalities are understood in an elementwise sense.}\\
	Let $\mathcal{X} = \{X_{ij} \in \perm_{m_im_j}\}_{i,j=1}^k$ be the set of pairwise matchings in a collection of $k$ objects, where each $(m_i {\times} m_j)$-dimensional matrix $X_{ij} \in \perm_{m_im_j}$ encodes the matching between the $m_i$ points in object $i$ and the $m_j$ points in object $j$.
	The set $\mathcal{X}$ is said to be \emph{cycle-consistent} if for all $i,j,\ell \in [k] := \{1,\ldots,k\}$ it holds that:
	\begin{enumerate}[label=(\roman*),noitemsep,topsep=0pt,parsep=2pt,partopsep=2pt]
		\item $X_{ii} = \matI_{m_i}$ (identity matching),
		\item $X_{ij} = X_{ji}^T$ (symmetry), and
		\item $X_{ij} X_{j\ell} \leq X_{i\ell}$ (transitivity).
	\end{enumerate}
\end{definition}
For the case of \emph{full} permutation matrices,~i.e.~in \eqref{eq:partPerm} the inequalities become equalities and $p{=}q$, Pachauri et al. \cite{Pachauri:2013wx} have proposed a simple yet effective method to achieve cycle-consistency based on a spectral decomposition of the matrix of pairwise matchings. The authors of \cite{Shen:2016wx} provide an analysis of such spectral synchronisations. Earlier works have also considered an iterative refinement strategy to improve pairwise matchings \cite{Nguyen:2011eb}.

While the aforementioned synchronisation methods considered the case of full permutations, some authors have also addressed the synchronisation of partial matchings,~e.g. based on semidefinite programming \cite{Chen:2014uo}, alternating direction methods of multipliers~\cite{zhou2015multi}, or a spectral decomposition followed by k-means clustering \cite{Arrigoni:2017ut}. A spectral approach has also been presented in~\cite{Maset:YO8y6VRb}, which, however, merely improves given initial pairwise matchings without guaranteeing cycle-consistency.

Rather than explicitly modelling the cubic number of cycle-consistency constraints (cf.~Def.~\ref{def:cycleCons}), most permutation synchronisation methods leverage the fact that cycle-consistency can be characterised by using the notion of \emph{universe points}, as e.g. in~\cite{Tron:kUBrCZhd}: %
\begin{lem} \label{lem:cycleConsRef} (Cycle-consistency, universe points)\\
The set $\mathcal{X}$ of pairwise matchings is \emph{cycle-consistent} if there exists a collection 
	$\{X_{\ell} \in \perm_{m_{\ell}d}~:~X_{\ell}\onevec_d = \onevec_{m_{\ell}}\}_{\ell=1}^k$
 such that for each $X_{ij} \in \mathcal{X}$ it holds that $X_{ij} = X_i X_j^T$.
\end{lem}
\begin{proof}
We show that conditions (i)-(iii) in Def.~\ref{def:cycleCons} are fulfilled. Let $i,j,\ell \in [k]$ be fixed.  (i) We have that $X_{i} \onevec_{d} = \onevec_{m_{i}}$ and 
$X_{i} \in \perm_{m_{i}d}$  imply that $X_i X_i^T = \matI_{m_i}$, so that $X_{ii} = X_iX_i^T = \matI_{m_i}$.
(ii) Moreover, $X_{ij} = X_i X_j^T$ means that $X_{ij}^T = (X_i X_j^T)^T = X_j X_i^T = X_{ji}$.
(iii) We have that %
$X_{j} \in \perm_{m_{j}d}$  implies 
$X_j^T X_j \leq \matI_{d}$.
We can write
\begin{align}
	X_{ij} X_{j\ell} &= (X_i X_j^T) (X_j X_{\ell}^T) \\
		&= X_i \underbrace{(X_j^T X_j)}_{\leq \matI_{d}} X_{\ell}^T 
		\leq X_i X_{\ell}^T  = X_{i\ell} \,. \qedhere
\end{align}

\end{proof}
Here, the  $(m_{\ell}{\times}d)$-dimensional \emph{object-to-universe} matching matrix $X_{\ell}$ assigns to each of the $m_{\ell}$ points of object $\ell$ exactly one of $d$ universe points---as such, all points among the $k$ objects that are assigned to a given (unique) universe point are said to be in correspondence. 

For notational brevity, it is convenient to consider a matrix formulation of Lemma~\ref{lem:cycleConsRef}. 
With $m_i$ being the number of points in the $i$-th object and $m = \sum_{i=1}^k m_i$, let $\mathbf{X}$ be the $(m {\times} m)$-dimensional pairwise matching matrix
\begin{align}
	\mathbf{X} = [X_{ij}]_{i,j=1}^k \in [\perm_{m_im_j}]_{i,j=1}^k %
\end{align}
and let
\begin{align}
	\mathcal{U} = \{ U  \in [\perm_{m_i d}]_{i}~:~ U \onevec_{d} = \onevec_m\} \subset \{0,1\}^{m \times d}\,.
\end{align}
Lemma~\ref{lem:cycleConsRef} translates into the requirement that there must be a $U \in \mathcal{U}$, such that
\begin{align}\label{eq:xuut}
	\mathbf{X} = UU^T\,.
\end{align}
With this matrix notation it becomes also apparent that one can achieve synchronisation by matrix \emph{factorisation}, such as pursued by the aforementioned spectral approaches \cite{Pachauri:2013wx,Shen:2016wx,Arrigoni:2017ut,Maset:YO8y6VRb}. While recently a lot of progress has been made for permutation synchronisation, one of the open problems is how to efficiently
integrate higher-order information to model geometric relations between points. We achieve this goal with our proposed method and demonstrate a significant improvement of the matching accuracy due to the additional use of geometric information. %

\paragraph{Power method:} 
The power method is one of the classical numerical linear algebra routines for computing the eigenvector corresponding to the largest (absolute) eigenvalue~\cite{GolubVanLoan:1996}.
 Moreover, it is well-known that the eigenvector corresponding to the largest (absolute) eigenvalue maximises the Rayleigh quotient $\frac{x^TAx}{x^Tx}$, which, up to scale, is equivalent to maximising the (not necessarily convex) quadrative objective $x^TAx$ over the unit sphere. In addition to computing a single eigenvector, straightforward extensions of the power method are the \emph{Orthogonal Iteration} and \emph{QR Iteration} methods~\cite{GolubVanLoan:1996}, which simultaneously compute multiple (orthogonal) eigenvectors and can be used to maximise quadratic objectives over the Stiefel manifold.

\paragraph{Power method generalisations:}  %
 Moreover, higher-order generalisations of the power method have been proposed, e.g.~for rank-1 tensor approximation~\cite{de1995higher}, or for multi-graph matching~\cite{Shi:2016tj}.
 However, the latter approach has a runtime complexity that is \emph{exponential} in the number of graphs and thus prevents scalability (e.g.~matching $12$ graphs, each with $10$ nodes, takes about $10$ minutes). In \cite{Chen:2016tx}, the authors propose a \emph{Projected Power Method} for the optimisation of quadratic functions over sets other than the Stiefel manifold, such as permutation matrices. 
 In a permutation synchronisations setting, their method obtains results that are comparable to semidefinite relaxations methods \cite{Chen:2014uo,Huang:2013uk} at a reduced runtime. However, due to the restriction to quadratic objective functions, their approach cannot handle geometric relations between points, as they would become polynomials of degree four, as will be explained in Sec.~\ref{sec:method}. Our method goes beyond the existing approaches as we propose a convergent  projected power iteration method for maximising a higher-order objective over the set $\mathcal{U}$. With that, we can incorporate geometric information between neighbouring points using a fourth-order polynomial, while always maintaining cycle-consistency. %

\section{Method}\label{sec:method}
The overall idea of our approach is to phrase the multi-matching problem as simultaneously solving $k^2$ pairwise matching problems with quadratic (second-order) matching scores, where the variables are weighted based on linear (first-order) matching scores. 
Instead of directly optimising over \emph{pairwise matchings}, we parametrise the pairwise matchings in terms of their \emph{object-to-universe} matchings,~cf.~Lemma~\ref{lem:cycleConsRef} and Eq.~\eqref{eq:xuut}.
Although this has the disadvantage that the quadratic term becomes quartic (a fourth-order polynomial), it has the strong advantages that (i) cycle-consistency is guaranteed to be always maintained, (ii) one only optimises for $m \times d$, rather than $m \times m$ variables in the pairwise case, where commonly $d \ll m$. %

\subsection{Multi-Matching Formulation}
In the following we present our multi-matching formulation.
Let $A_i \in \mathcal{S}_{m_i}^+$ denote the $(m_i {\times} m_i)$-dimensional symmetric and positive semidefinite \emph{adjacency matrix} of object $i$ (e.g.~a matrix that encodes the Gaussian of pairwise Euclidean/geodesic distances between pairs of points).
For a given $X_i \in \perm_{m_i d}$, the $d {\times} d$ matrix $X_i^T A_i X_i$ is a row/column reordering of the matrix $A_i$ according to the universe points.
Hence, we can use the Frobenius inner product $\langle X_i^T A_i X_i, X_j^T A_j X_j\rangle$  for quantifying how well two adjacency matrices $A_i$ and $A_j$ agree after they have been reordered according to the universe points based on $X_i$ and $X_j$. This term can be understood as the object-to-universe formulation of second-order matching terms when using pairwise matching matrices (analogous to the QAP in Koopmans-Beckmann form~\cite{Koopmans:1957gf}).  As such, multi-matching with geometric consistency can be phrased as
\begin{align}
  &\max_{X_1,\ldots,X_k} &&\sum_{i,j=1}^k  \langle X_i^T A_i X_i, X_j^T A_j X_j\rangle \nonumber \\
  & \quad\text{s.t} & & X_i \in \perm_{m_id}~\forall~i \in [k]\,. \label{eq:mgmpwonly}
\end{align}
Equivalently, for $U = [X_1^T, \ldots, X_k^T]^T \in \R^{m \times d}$ and $A \in \mathcal{S}_{m}^+$ being the (symmetric and positive semidefinite) block-diagonal \emph{multi-adjacency} matrix defined as $A = \diag(A_1,\ldots,A_k) \in \R^{m \times m}$, Problem~\eqref{eq:mgmpwonly} can be written in compact matrix form as
\begin{align}
  &\max_{U \in \mathcal{U}} && \trace(U^T A UU^T A U) \,. \label{eq:mgmmatrixpwonly}
\end{align}
Let $W_{ij} \in \R^{m_i \times m_j}_+$ encode the (non-negative) similarity scores between the points of object $i$ and $j$,
 and let $W = [W_{ij}]_{ij} \in \R^{m \times m}$ be a (symmetric) matrix that encodes all similarity scores. With that, we can define the reweighting of the object-to-universe matchings $U$ as $WU$, where both $U$ and $WU$ have the same dimensionality. The purpose of this reweighting is to amplify matchings in $U$ that have high similarity scores, cf.~Fig.~\ref{fig:reweighting}.
By replacing $U$ in Problem~\eqref{eq:mgmmatrixpwonly} with its reweighted matrix $WU$, we arrive at our final multi-matching formulation
\begin{align}
  &\max_{U \in \mathcal{U}} && \trace( U^T \overline{W} U U^T \overline{W} U) := f(U) \,, \label{eq:mgm}
\end{align}
for $\overline{W} := W^T A W$.
While~\eqref{eq:mgm} is based on the $U$-matrix and hence intrinsically guarantees cycle-consistent multi-matchings, the objective function is a fourth-order polynomial that is to be maximised over the (binary) set $\mathcal{U}$.
\begin{figure}%
     \centerline{  
        \includegraphics[scale=.22]{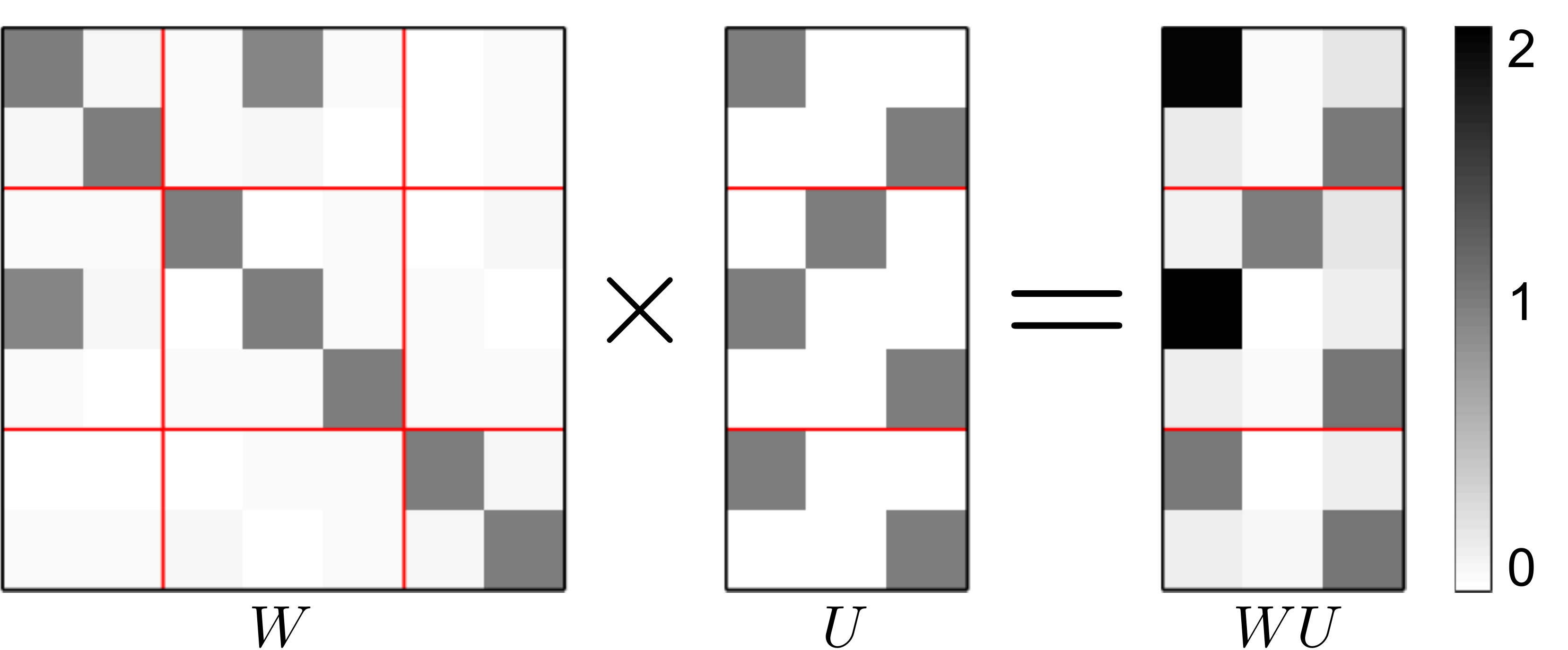}%
     }
    \caption{Effect of reweighting $U$ with $W$. The similarity matrix $W$ indicates that the first point in the first object is similar the second point in the second object (the grey values in $W$ at position $(1,4)$, and $(4,1)$). Hence, those elements in the first column of $U$that match these points to each other are amplified in $WU$.}
    \label{fig:reweighting}
\end{figure}

\subsection{Algorithm}
In order to solve Problem~\eqref{eq:mgm} we propose a higher-order projected power iteration (HiPPI) method, as outlined in Alg.~\ref{alg:hppi}. The main idea is to perform a higher-order power iteration step, followed by a projection onto the set $\mathcal{U}$. The approach is simple, and merely comprises of matrix multiplications and the Euclidean projection onto $\mathcal{U}$.
\begin{algorithm}\label{alg:hppi}
\SetKwInput{Input}{Input}
\SetKwInput{Output}{Output}
\SetKwInput{Initialise}{Initialise}
\SetKwRepeat{Do}{do}{while}
\DontPrintSemicolon

 \Input{similarities $W$, multi-adjacency matrix $A {\in} \mathcal{S}_m^+$}
 \Output{cycle-consistent multi-matching $U_t \in \mathcal{U}$}
\Initialise{$t \gets 0, U_0 \in \mathcal{U}, \overline{W} \gets W^T A W$}%

  \Repeat{ $|f(U_t) - f(U_{t{-}1}) | = 0$}{ \label{alg:line1} %
  	$V_t \gets \overline{W}U_tU_t^T\overline{W} U_t$ \label{alg:quartic}  \\
  	
  	$U_{t{+}1} \gets \operatorname{proj}_{\mathcal{U}}(V_t)$  \label{alg:proj}\\
    $t \gets t{+}1$
   }
 \caption{Higher-order Projected Power Iteration (HiPPI) algorithm.}
\end{algorithm}

\subsection{Theoretical Analysis}
\paragraph{Monotonicity:} We now show that the HiPPI algorithm monotonically increases the objective $f(U_t)$.
\begin{prop}\label{prop:mon}%
In Alg.~\ref{alg:hppi} it holds that $f(U_{t{+}1}) \geq f(U_t)$.
\end{prop}
\noindent Before we prove Prop.~\ref{prop:mon}, we state the following result:
\begin{lem}\label{lem:nondec}
  Let $W$ be a symmetric and positive semidefinite matrix. For $\trace(V^TWVV^TWV) \leq \trace(U^TWVV^TWU)$ we have $\trace(U^TWVV^TWU) \leq \trace(U^TWUU^TWU)$.
 \end{lem}

   \begin{proof}
    With $W$ being positive semidefinite, we can factorise it as $W=L^TL$. We have
    \begin{align}
      0& \leq \|LUU^TL^T - LVV^TL^T\|^2 \\
      &= \trace(LUU^TL^TLUU^TL^T) + \trace(LVV^TL^TLVV^TL^T)  \nonumber\\
      &\qquad - 2\trace(LUU^TL^T LVV^TL^T) \\
      &= \trace(U^TWUU^TWU)  {-} \trace(U^TWVV^TWU)  \\
      &\qquad\underbrace{- \trace(U^TWVV^TWU) + \trace(V^TWVV^TWV)}_{\leq 0 \text{ by assumption}} \nonumber \\
      \Rightarrow& \trace(U^TWUU^TWU)  {-} \trace(U^TWVV^TWU) \geq 0 \,.\qedhere
    \end{align}
   \end{proof}

\begin{proof}[Proof of Prop.~\ref{prop:mon}]
We first look at a simpler variant of Problem~\eqref{eq:mgm}, where for a fixed $Y$ we fix the inner term $UU^T$ in $f$ to $YY^T$, define $Z:= \overline{W} YY^T \overline{W}$, and relax $\mathcal{U}$ to its convex hull $\mathcal{C} := \operatorname{conv}(\mathcal{U})$, so that we obtain the problem
\begin{align}
  \max_{U \in \mathcal{C}}  ~\trace( U^T Z U) 
  ~~\Leftrightarrow~~ \min_{U} ~ \underbrace{\iota_{\mathcal{C}}(U)}_{g(U)} - \underbrace{\trace( U^T Z U)}_{h(U)} \,. \label{eq:dc}
\end{align}
Here, $\iota_{\mathcal{C}}(U)$ is the (convex) indicator function of the set $\mathcal{C}$ and $\trace( U^T Z U)$ is a convex quadratic function. We can see that~\eqref{eq:dc} is in the form of a difference of convex functions and can thereby be tackled based on DC programming, which for a given initial $U_0$ repeatedly applies the following update rules (see~\cite{le2018dc}): %
\begin{align}
      V_t &= \nabla_{U} h(U_t) = 2Z U_t \\
      U_{t{+}1} &= \argmin_U ~g(U) - \langle U, V_t \rangle \\
                &= \argmin_{U \in \mathcal{C}}  ~{-}\langle U, V_t \rangle \,.
\end{align}
As the maximum of a linear objective over the  compact convex set $\mathcal{C}$ is attained at its extreme points $\mathcal{U}$, 
we get
\begin{align}
                U_{t{+}1} &= \argmin_{U \in \mathcal{U}}  ~{-}\langle U, V_t \rangle \label{eq:projislap}\\
                &= \argmin_{U \in \mathcal{U}} \|V_t - U \|^2_F = \operatorname{proj}_{\mathcal{U}}(V_t)\,,
  \end{align}
  where for the latter we used that $\langle U,U \rangle {=} m$ (since any $U \in \mathcal{U}$ is a binary matrix that has exactly a single element in each row that is $1$). Based on the descent properties of DC programming (i.e.~the sequence $(g(U_t)-h(U_t))_{t=0,1,\ldots}$ is decreasing, cf.~\cite{le2018dc}, and thus $(h(U_t))_{t=0,1,\ldots}$ is increasing), so far we have seen that when applying the update $U_{t{+}1} = \operatorname{proj}_{\mathcal{U}}(\overline{W} YY^T \overline{W}U_t)$ we get that
  \begin{align}
   \trace(U_t^T \overline{W} YY^T \overline{W} U_t) \leq \trace(U_{t{+}1}^T \overline{W} YY^T \overline{W}  U_{t{+}1}) \,.
   \end{align}
   In particular, this also holds for the choice $Y:= U_t$, i.e.
   \begin{align}
   f(U_t) &= \trace(U_t^T \overline{W} U_tU_t^T \overline{W} U_t) \\
          &\leq \trace(U_{t{+}1}^T \overline{W} U_tU_t^T \overline{W}  U_{t{+}1}) \,.
  \end{align}
  Since $A$ is positive semidefinite by assumption, $\overline{W} = W^TAW$ is also positive semidefinite, and therefore we can apply Lemma~\ref{lem:nondec} to get
  \begin{align}
   f(U_t) &\leq\trace(U_{t{+}1}^T \overline{W} U_tU_t^T \overline{W}  U_{t{+}1})  \\
    &\leq\trace(U_{t{+}1}^T \overline{W} U_{t{+}1}U_{t{+}1}^T \overline{W}  U_{t{+}1}) = f(U_{t{+}1})\,. \qedhere
  \end{align}

\end{proof}

\paragraph{Convergence:}
For the sake of completeness, we also provide the following corollary.
\begin{cor}The sequence $(f(U_t))_{t=0,1,\ldots}$ produced by Alg.~\ref{alg:hppi} converges after a finite number of iterations.
\end{cor}
\begin{proof}
Since $\mathcal{U}$ is a finite set, $f(U)$ is  bounded above for any $U \in \mathcal{U}$. Moreover, since for any $t \geq 0$ we have that $U_t \in \mathcal{U}$ (feasibility), the sequence $(f(U_t))_{t=0,1,\ldots}$ produced by Alg.~\ref{alg:hppi} is bounded and increasing (Prop.~\ref{prop:mon}), and hence convergent. Since the $U_t$ are discrete, convergence implies that there exists a $t_0 \in \N$ such that for all $ t \geq t_0$ it holds that $f(U_t) = f(U_{t_0})$. %
\end{proof}

\begin{remark} 
  As the order of the universe points can be arbitrary, there is a whole family of equivalent solutions: for any $U' = UP$ for $P$ being a $d {\times} d$ permutation, we have that $f(U) = f(U')$, since $U'(U')^T = UPP^TU^T = UU^T$. 
\end{remark}

\paragraph{Complexity analysis:}
The update rule in Alg.~\ref{alg:hppi} can be written as $U_{t{+}1} = \operatorname{proj}_{\mathcal{U}}(\overline{W} U_tU_t^T \overline{W}U_t)$. The matrix multiplications for computing $\overline{W} U_tU_t^T \overline{W}U_t$ have time complexity $\mathcal{O}(m^2d)$.
As can be seen in the proof of Prop.~\ref{prop:mon}, the projection onto $\mathcal{U}$ is a linear programming problem. Finding its optimiser amounts to solving $k$ individual (partial) linear assignment problems, each having sub-cubic empirical average time complexity when using the Auction algorithm~\cite{Bertsekas:1998vt,Schwartz:1994db}. %
Hence, the overall (average) per-iteration complexity is $\mathcal{O}(m^2d + k d^2\log(d))$. The memory complexity is $\mathcal{O}(m^2)$ due to the matrix $W \in \R^{m \times m}$, which can be improved by considering \emph{sparse} similarity scores.

\newcommand{\figScale}{.75}
\begin{figure*}%
     \centerline{  
        \subfigure{\includegraphics[scale=\figScale]{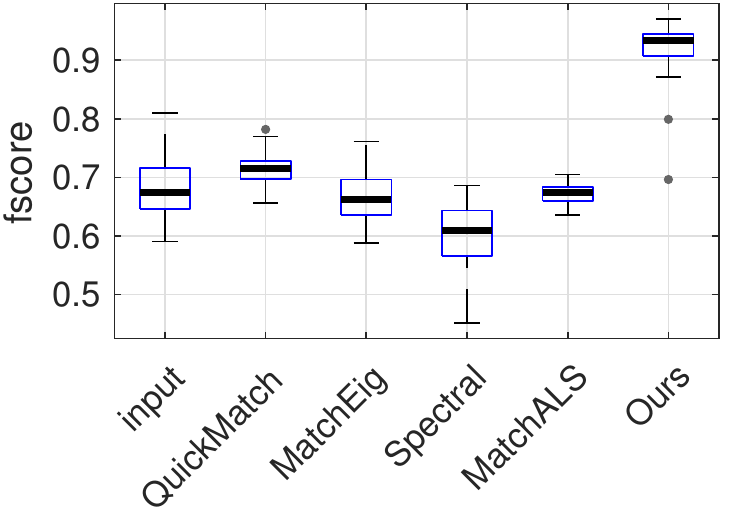}}%
        \hfil
        \subfigure{\includegraphics[scale=\figScale]{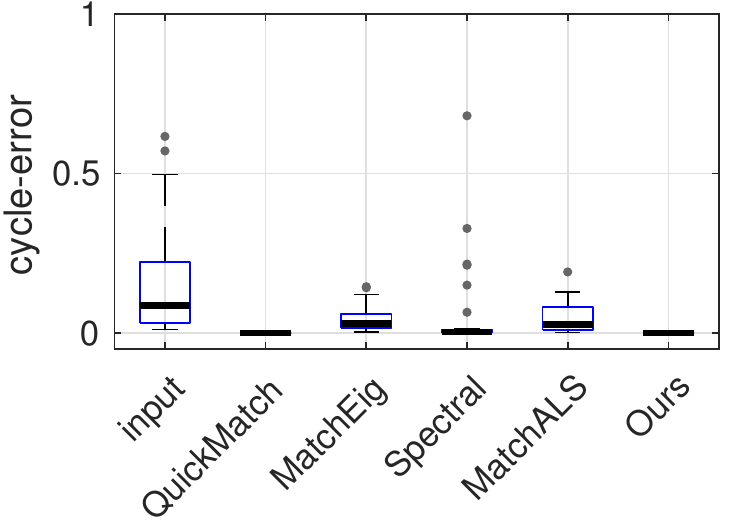}}%
        \hfil
        \subfigure{\includegraphics[scale=\figScale]{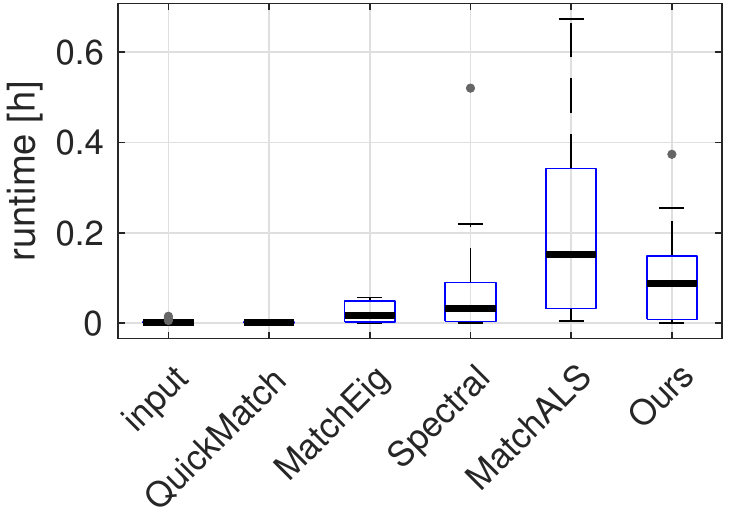}}%
     }
     \vspace{-1mm}
    \caption{Multi-image matching on the HiPPI dataset. Our method is the only one that considers geometric relations between points, while at the same time being scalable to such large datasets (note that the reported runtimes include the initialisation). We clearly outperform other matching methods as well as synchronisation methods in terms of the fscore (higher is better), while also guaranteeing cycle-consistency. } 
    \label{fig:resultsJohanMim}
\end{figure*}

\section{Experimental Results}\label{sec:experiments}

In this section we extensively compare our method to other approaches on three datasets. To be more specific, we consider two multi-image matching datasets, Willow~\cite{cho2013} and HiPPI, as well as the multi-shape matching dataset Tosca~\cite{bronstein2008numerical}. The datasets are summarised in Table~\ref{table:datasets}.
\begin{table}[h!]
\vspace{-2mm}
\begin{tabular}{@{}lllllll@{}}
\toprule
\textbf{Dataset} & \textbf{Type}   & \textbf{Bij.} & \textbf{\#} & $\mathbf{k}$ & $\mathbf{m}$  \\ \midrule
HiPPI   & images & no & $31$                   &   $[7{,}100]$  &  $[2257{,}20703]$    \\
Willow  & images & yes & $5$                   &   $[40{,}108]$  &  $[400{,}1080]$        \\
Tosca   & shapes & no & $7$                   &   $[6{,}20]$  &     $[3000{,}10000]$    \\ \bottomrule
\end{tabular}
\caption{Overview of datasets. ``Bij.'' indicates whether the matchings are bijective, and ``\#'' is the number of instances per dataset.
 }
\label{table:datasets}
\vspace{-2mm}
\end{table}

\paragraph{Similarity scores:} The similarity scores between image/shape $i$ and $j$ is encoded in the matrix $W_{ij} \in \R^{m_i \times m_j}_+$, which is defined using \emph{feature matrices} $F_i \in \R^{m_i \times f}$ and $F_j \in \R^{m_j \times f}$ of the respective image/shape, where $f$ is the feature dimensionality. As in~\cite{Tron:kUBrCZhd}, the similarity scores are based on a weighted Gaussian kernel, i.e.
\begin{align}
	[W_{ij}]_{pq} = \omega_{pq} \exp(- \frac{\|[F_i]_{p} - [F_j]_q\|^2}{2\sigma^2})\,,
\end{align}
where $\omega_{pq}$ is a weight that depends on the distance between the features and the closest descriptor from the same image. For details we refer the reader to~\cite{Tron:kUBrCZhd}.
The particular choice of features for each dataset are described below.

\paragraph{Adjacency matrices:} The adjacency matrix $A_i \in \mathcal{S}_{m_i}^+$ of image/shape $i$ is based on Euclidean distances between pairs of 2D image point locations in the case of multi-image matching (or geodesic distances between pairs of points on the 3D shape surface for multi-shape matching). By denoting the distance between the points with indices $p {\in} [m_i]$ and $q {\in} [m_i]$ as $d_{pq}$, the elements of the adjacency matrix are based on a Gaussian kernel, so that
	$[A_i]_{pq} {=} \exp(- \frac{d_{pq}^2}{2\mu\sigma_A^2})$.
We set $\sigma_A {=} \operatorname{median}(d_{\min})$, where $[d_{\min}]_p {=} \min_{q\neq p} d_{pq}$ for $p\in[m_i]$, and $\mu$ is a scaling factor.

\begin{table*}[!h!t]%
    \begin{center}%
      \footnotesize{%
      \setlength\tabcolsep{10pt}%
      \begin{tabular}{@{}llllllll@{}}
\toprule
\textbf{Instance} & \textsc{Pairwise} & \textsc{Spectral} & \textsc{MatchALS} & \textsc{QuickMatch} & \textsc{Yan et al.} & \textsc{Wang et al.} & \textsc{Ours}\\ \midrule
Car & $0.54~(0.4s)$ & $0.65~(0.4s)$ & $0.61~(4.2s)$ & $0.24~(0.2s)$ & $0.57~(21.5s)$ & $0.71~(2.9s)$ & $\mathbf{0.74}~(0.6s)$ \\
Duck & $0.48~(0.6s)$ & $0.66~(0.6s)$ & $0.62~(3.8s)$ & $0.23~(0.1s)$ & $0.54~(20.5s)$ & $0.82~(3.8s)$ & $\mathbf{0.88}~(0.8s)$ \\
Face & $0.94~(3.2s)$ & $0.98~(3.3s)$ & $0.96~(6.3s)$ & $0.90~(0.2s)$ & $0.85~(13.7s)$ & $0.96~(4.5s)$ & $\mathbf{1.00}~(3.4s)$ \\
Motorbike & $0.33~(0.4s)$ & $0.37~(0.5s)$ & $0.32~(1.8s)$ & $0.13~(0.2s)$ & $0.32~(31.8s)$ & $0.65~(8.2s)$ & $\mathbf{0.84}~(0.8s)$ \\
Winebottle & $0.62~(1.1s)$ & $0.82~(1.2s)$ & $0.77~(4.2s)$ & $0.24~(0.1s)$ & $0.71~(18.0s)$ & $0.88~(3.9s)$ & $\mathbf{0.95}~(1.4s)$ \\
\bottomrule
\end{tabular}
      }%
    \end{center}%
    \vspace{-3mm}%
    \caption{Fscores (higher is better) and runtimes for the Willow dataset. All methods are initialised based on pairwise matches with linear costs (except \textsc{QuickMatch}~\cite{Tron:kUBrCZhd}, which is initialisation-free). \textsc{Ours} ($\mu{=}{10}$), \textsc{Yan et al.}~\cite{Yan:2016vf}, and \textsc{Wang et al.}~\cite{Wang:2017ub} consider geometric relations between points, amongst which  \textsc{Ours} is the fastest (note that the reported runtimes include the initialisation). 
    }
    \label{table:willow}
\end{table*}

\paragraph{Quantitative scores:} For the quantitative evaluation we consider the $\text{fscore} = 2{\cdot}\frac{\text{precision} \cdot \text{recall}}{\text{precision} + \text{recall}}$, and the cycle-error, which is given by the fraction of the total number of cycle-consistency violations in all three-cycles, divided by the total number of matchings in all three-cycles.

\subsection{HiPPI Dataset} 
In this experiment we compare various multi-image matching methods. 

\paragraph{Dataset:} The HiPPI dataset %
comprises $31$ multi-image matching problems. For each problem instance, a (short) video sequence has been recorded (with resolution~$1920 {\times} 1080$, frame-rate~$30$ FPS, and duration~${>}5$s).
In each video, feature points and feature descriptors were extracted using SURF~\cite{Bay:2008jv} with three octaves. 
To obtain ground truth matchings, these feature points were tracked across the sequence based on their geometric distance and feature descriptor similarity. 
To ensure \emph{reliable} ground truth matchings, we have conducted the following three steps: (i) obvious wrong matchings between consecutive images have been automatically pruned,
(ii) we have manually removed those features that were incorrectly tracked from the first to the last frame (by inspecting the first and last frame), and 
(iii) in order to prevent feature sliding in-between the first and last frame, we manually inspected the flow of each remaining feature point and removed wrongly tracked points. Note that steps (ii) and (iii) have been performed by two different persons, which took in total about $24$ hours.
A multi-matching problem was then created by extracting evenly spaced frames from a sequence, where in each frame we added a significant amount of outlier points (randomly selected from the previously pruned points, where the number of points is chosen such that in each frame $50\%$ of the points are outliers), and we simulate occlusions (a rectangle of size $0.2 {\times} 0.2$ of the image dimensions) in order to get difficult partial multi-matching problems. 

\paragraph{Multi-image matching:} %
We compare our method with \textsc{QuickMatch} \cite{Tron:kUBrCZhd}, \textsc{MatchEig} \cite{Maset:YO8y6VRb}, \textsc{Spectral} \cite{Pachauri:2013wx} (implemented by the authors of \cite{zhou2015multi} for \emph{partial} permutation synchronisation), and \textsc{MatchALS} \cite{zhou2015multi}. 
Other methods that incorporate geometric information~\cite{Yan:2015vc,Yan:2016vf,Wang:2017ub} do not scale to such large problem instances, see Fig.~\ref{fig:teaser}, and thus we cannot compare to them.
The universe size $d$ is set to twice the average of the number of points per image~\cite{Maset:YO8y6VRb}.
We used \textsc{QuickMatch} ($\rho_{\text{den}} {=} 0.7$) for initialising $U_0$ in our method, and we set $\mu {=} 1$. The results are shown in Fig.~\ref{fig:resultsJohanMim}, where it can be seen that our method achieves a superior matching quality. In contrast to other methods (except \textsc{QuickMatch}), our method guarantees cycle-consistency.

\newcommand{\figScaleC}{.5}

\begin{figure*}[t!p!]
\centerline{\includegraphics[scale=\figScaleC]{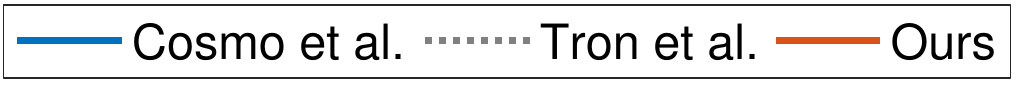}}
\vspace{-2mm}
     \centerline{ 
        \subfigure{\includegraphics[scale=\figScaleC]{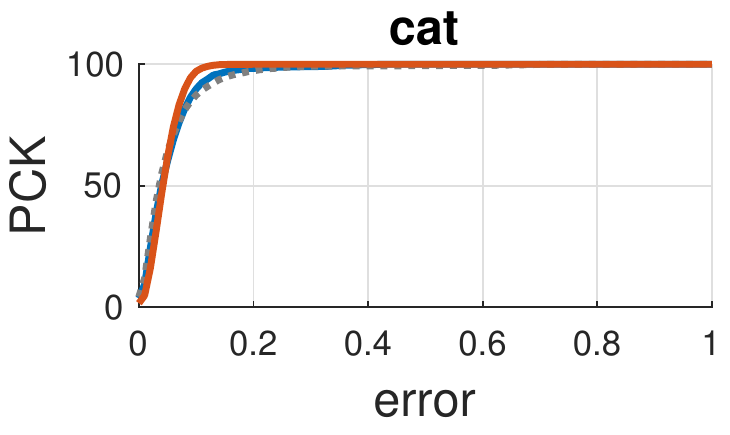}} \hfil
        \subfigure{\includegraphics[scale=\figScaleC]{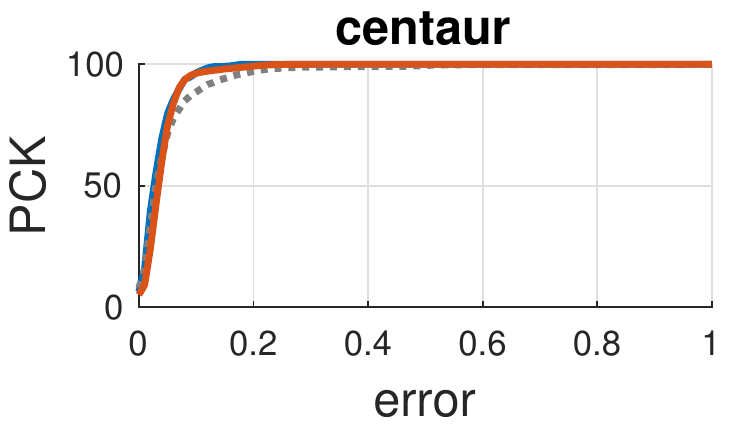}} \hfil
        \subfigure{\includegraphics[scale=\figScaleC]{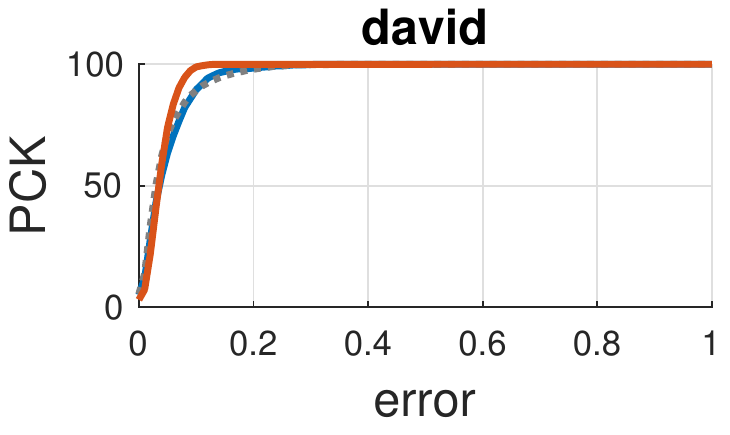}} \hfil
        \subfigure{\includegraphics[scale=\figScaleC]{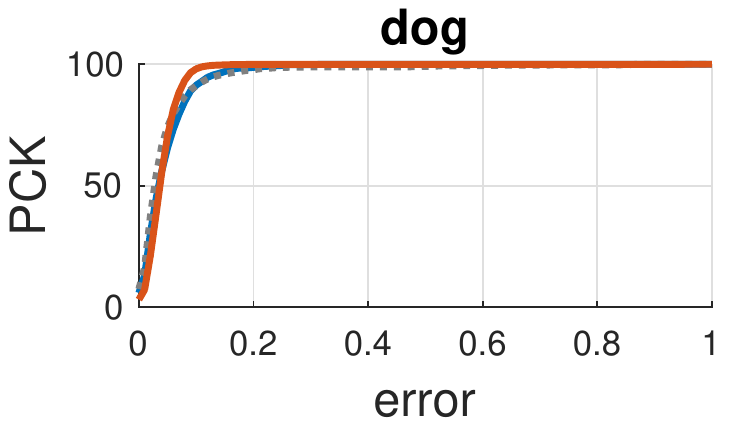}} 
        }
        \vspace{-3mm}
    \centerline{
        \subfigure{\includegraphics[scale=\figScaleC]{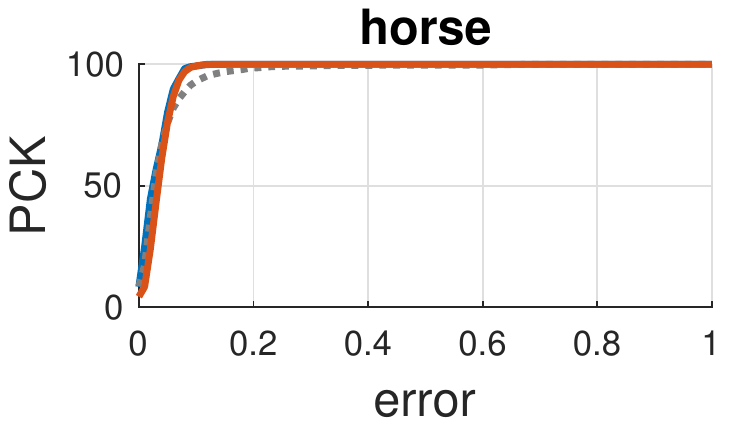}} \hfil
        \subfigure{\includegraphics[scale=\figScaleC]{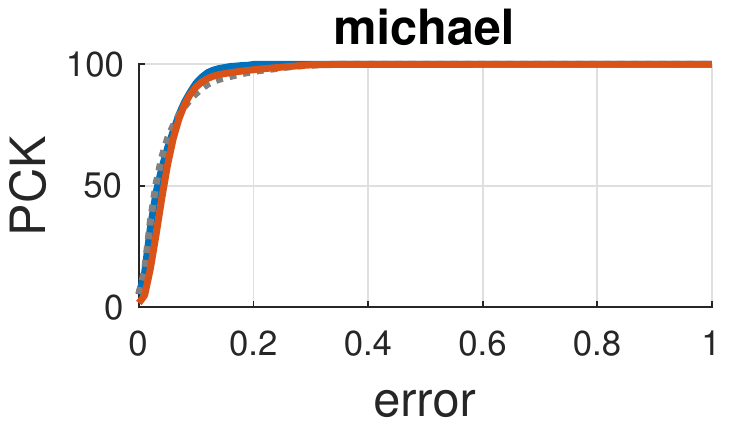}} \hfil        
        \subfigure{\includegraphics[scale=\figScaleC]{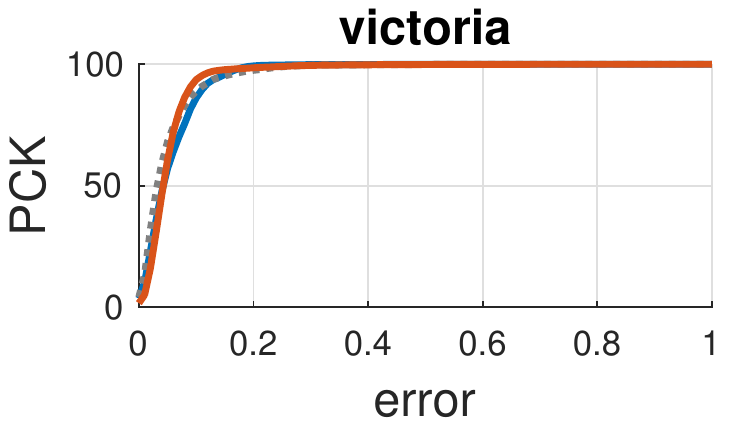}} \hfil
        \begin{subtable}{~~~~~~\scriptsize\begin{tabular}[b]{@{}lrr@{}}
\toprule
\#matches &\textbf{\cite{cosmo2017consistent}} &\textbf{Ours}\\  
 \midrule 
cat & $88$ & $\mathbf{518}$ \\
centaur & $51$ & $\mathbf{551}$ \\
david & $136$ & $\mathbf{515}$ \\
dog & $137$ & $\mathbf{582}$ \\
horse & $21$ & $\mathbf{515}$ \\
michael & $83$ & $\mathbf{511}$ \\
victoria & $120$ & $\mathbf{515}$ \\
\bottomrule
\vspace{-2.5mm}\end{tabular}
}~~\end{subtable}
     }
     \vspace{-1.5mm}
    \caption{Our method obtains significantly more multi-matchings (bottom right, see also Fig.~\ref{fig:qualitativeResultsTosca}) compared to the method of Cosmo~et~al.~\cite{cosmo2017consistent}, while at the same time achieving comparable errors (percentage of correct keypoints, PCK).} 
    \label{fig:resultsTosca}
\end{figure*} 
\newcommand{\figScaleD}{.4}
\newcommand{\figScaleE}{.45}
\newcommand{\figScaleFF}{.31}
\begin{figure}[h!] 
  \rotatebox[origin=l]{90}{$~~~$ \textsc{QuickMatch} \cite{Tron:kUBrCZhd}} 
  \vspace{-5mm}
     \centerline{ 
        \subfigure{\includegraphics[scale=\figScaleE]{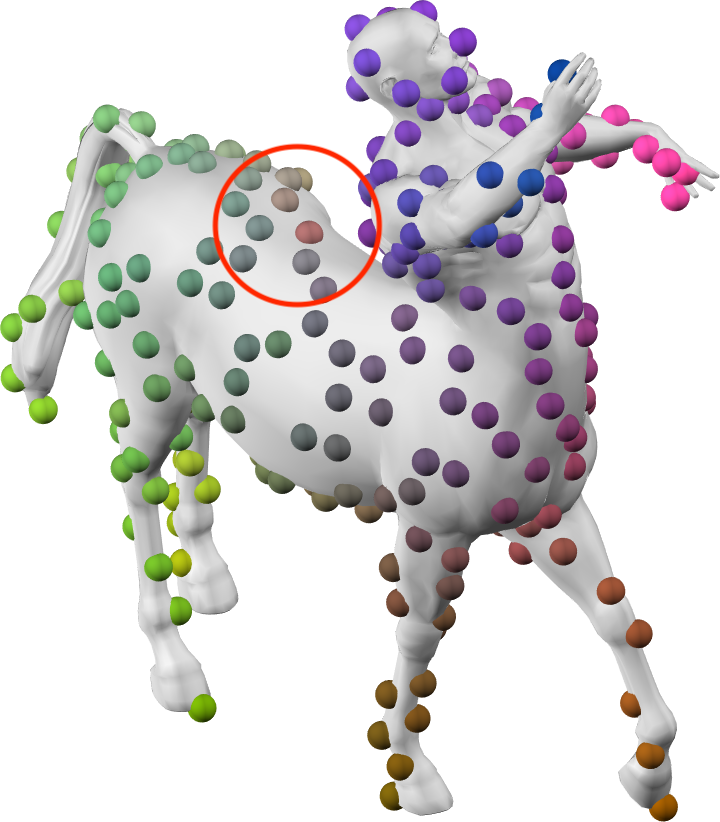}}%
        \subfigure{\includegraphics[scale=\figScaleE]{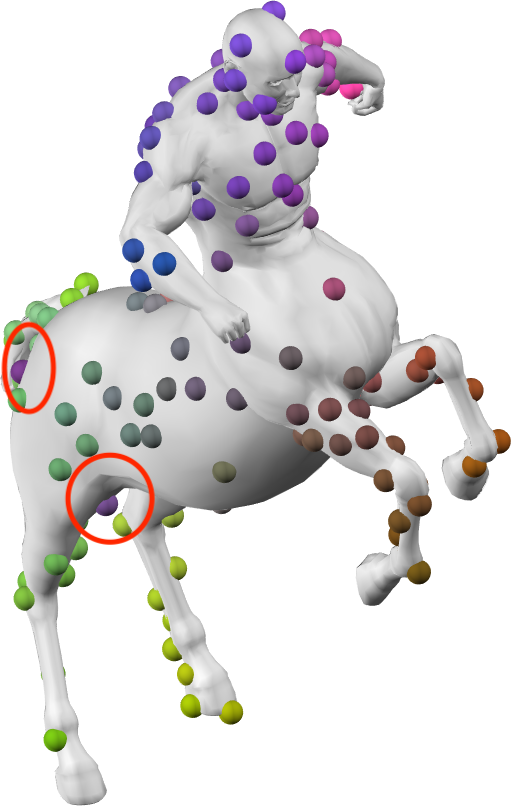}}%
        \subfigure{\includegraphics[scale=\figScaleFF]{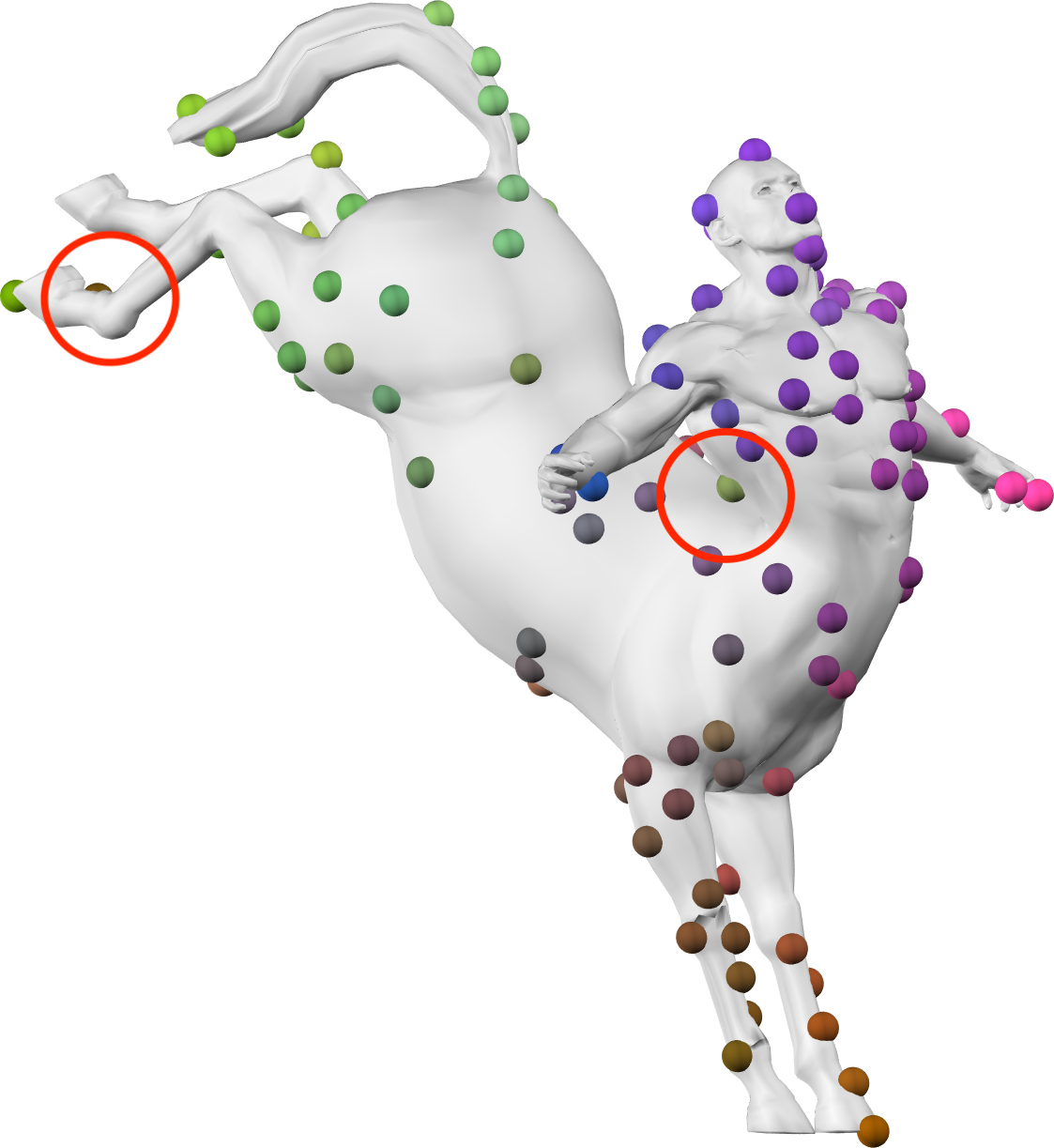}}%
     }
     \vspace{-5mm}
  \rotatebox[origin=l]{90}{$~~~$ \textsc{Cosmo et al.} \cite{cosmo2017consistent}}%
     \centerline{
        \subfigure{\includegraphics[scale=\figScaleE]{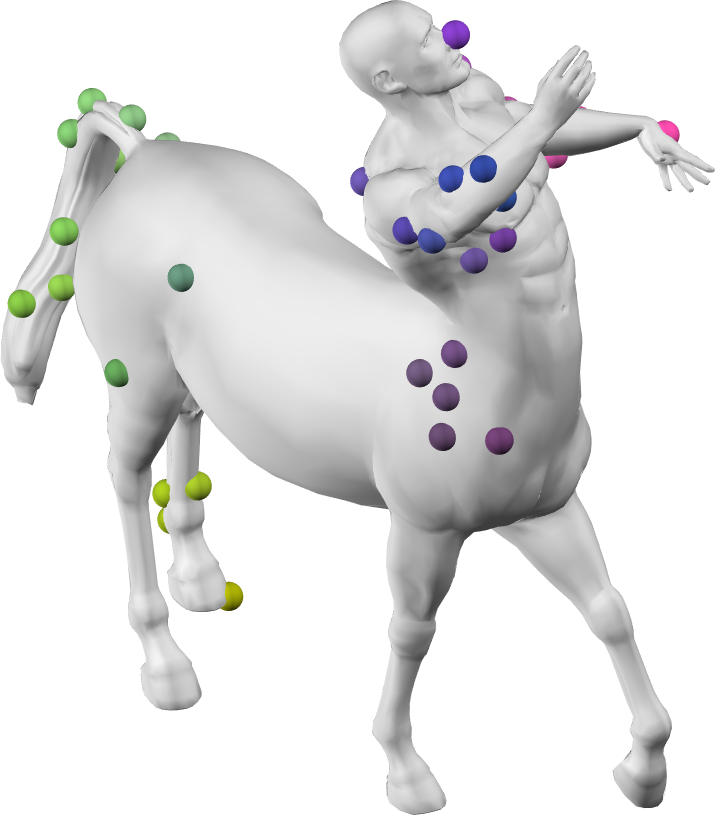}}%
        \subfigure{\includegraphics[scale=\figScaleE]{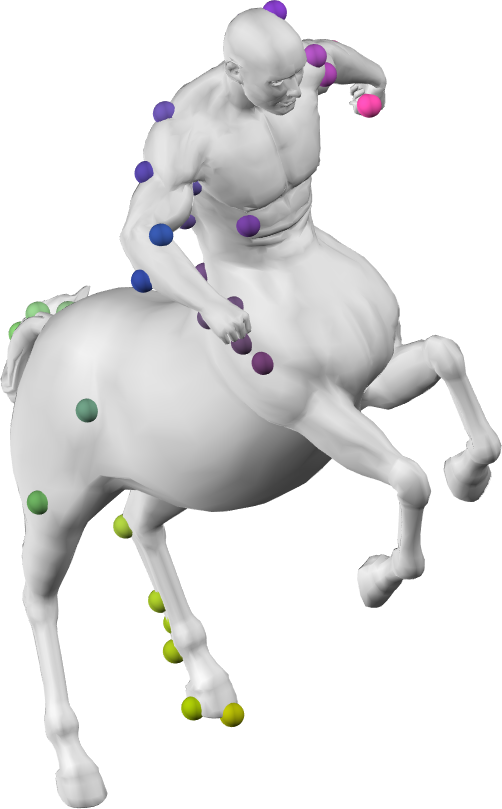}}%
        \subfigure{\includegraphics[scale=\figScaleFF]{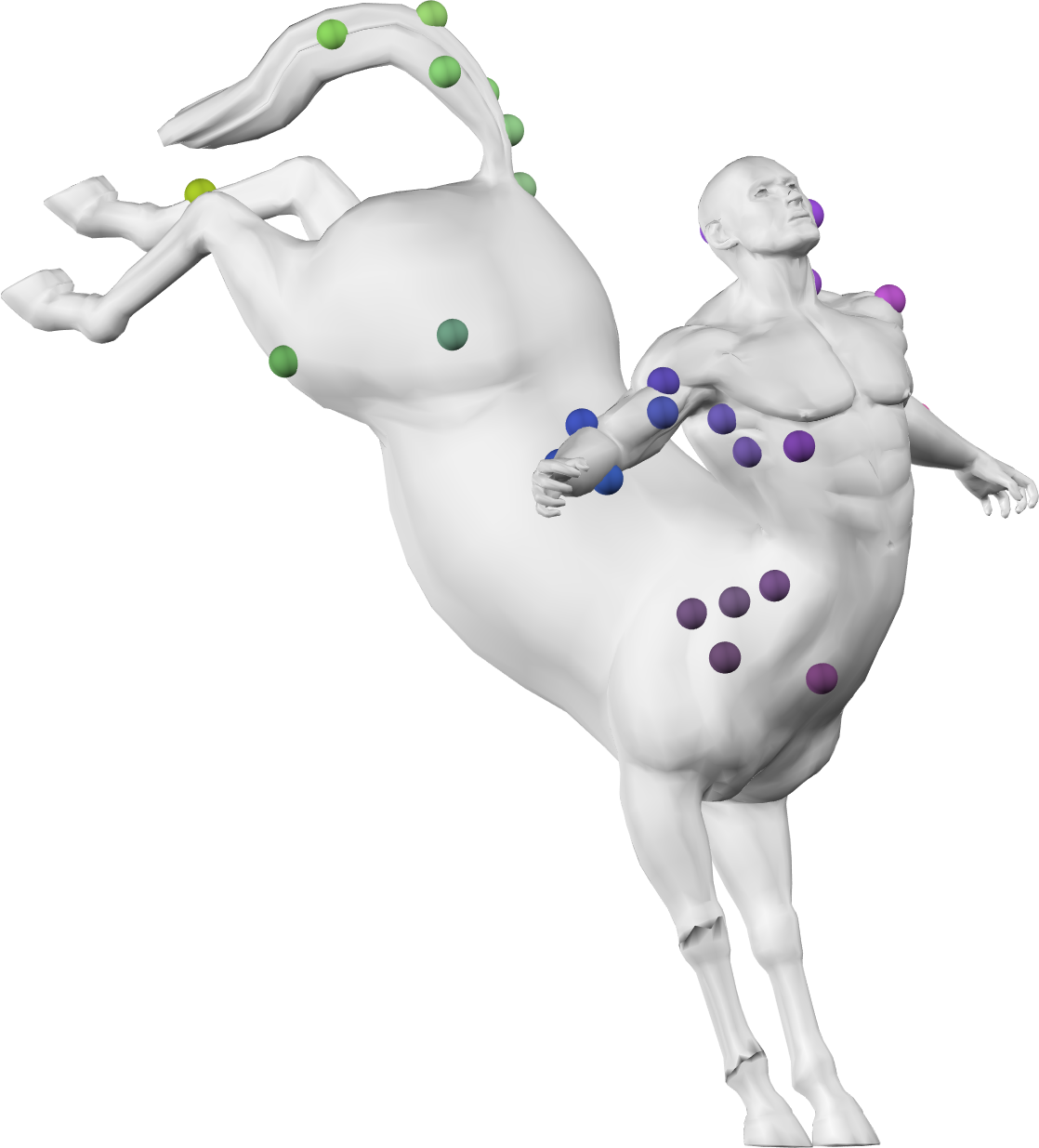}}%
     }
     \rotatebox[origin=l]{90}{$\quad\qquad$ \textsc{Ours}}
     \centerline{
        \subfigure{\includegraphics[scale=\figScaleE]{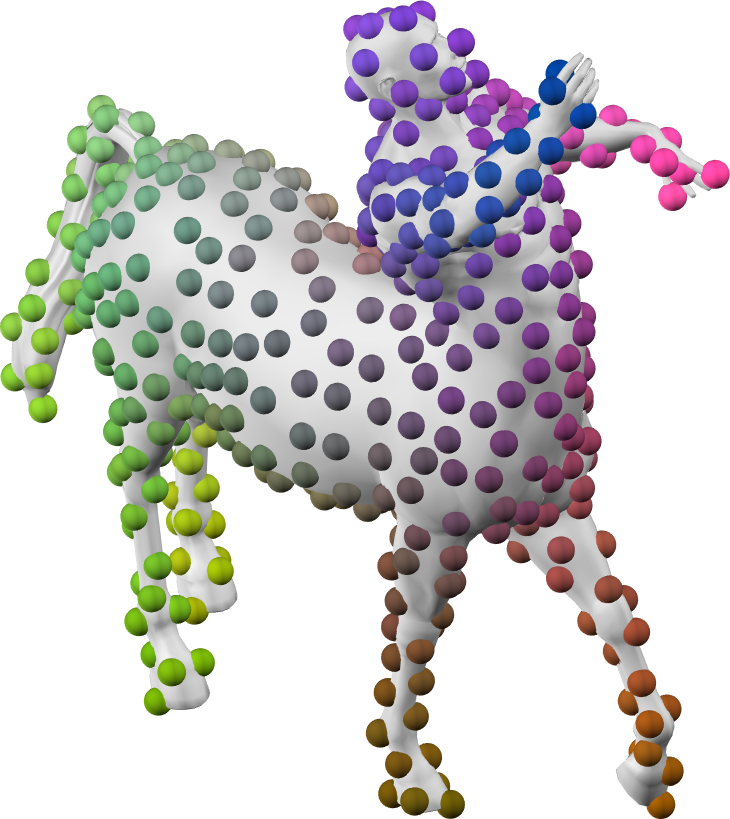}}%
        \subfigure{\includegraphics[scale=\figScaleE]{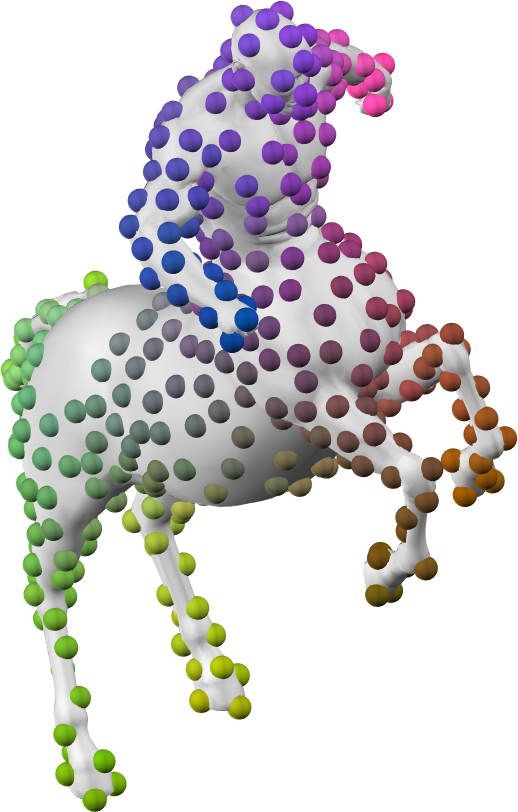}}%
        \subfigure{\includegraphics[scale=\figScaleFF]{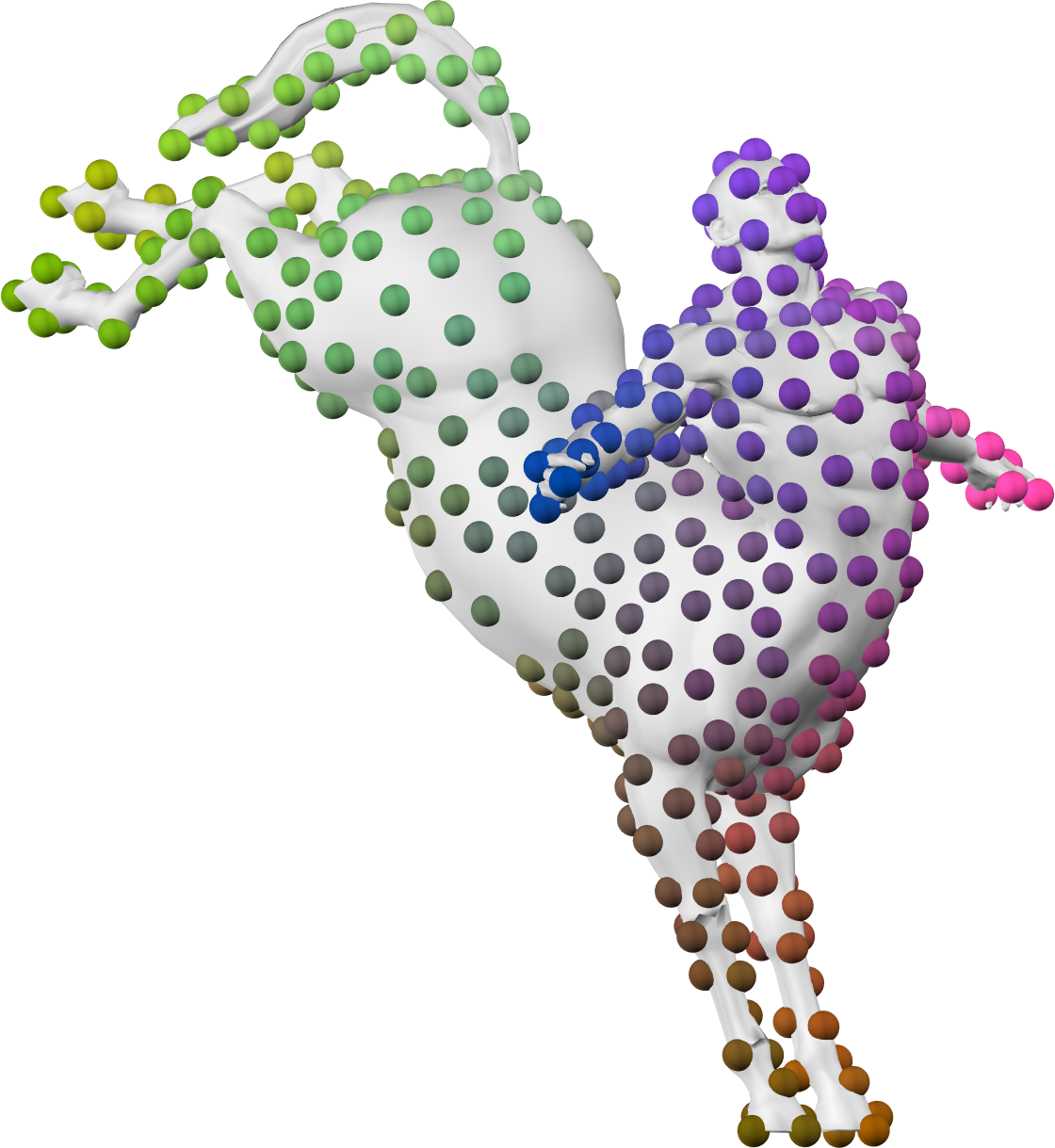}}%
     }
     \vspace{-5.5mm}
    \caption{Qualitative results on the Tosca \emph{centaur}. Dots with same colour indicate matched points. \textsc{QuickMatch} (top) does not take geometric consistency into account and thus leads to mismatches (red circles). While  Cosmo et al.~\cite{cosmo2017consistent} (centre) obtain only \emph{few} reliable matchings, leading to \emph{regions without correspondences}, our method (bottom) obtains \emph{significantly more reliable matchings}.} 
    \label{fig:qualitativeResultsTosca}
    \vspace{-1mm}
\end{figure} 

\subsection{Willow Dataset}
The evaluation on the Willow dataset~\cite{cho2013} is based on the experimental protocol from~\cite{Wang:2017ub}, where deep features have been used for matching.
 For this dataset the matchings are bijective, and hence for all methods we set the universe size $d$ to the number of annotated features. Since \textsc{QuickMatch}~\cite{Tron:kUBrCZhd} is tailored towards partial matchings, as it implicitly determines the universe size during its internal clustering, we have found that it does not perform so well on this dataset (see Table~\ref{table:willow}). 
In Table~\ref{table:willow} it can be seen that our method is superior compared to the other approaches.

\subsection{Tosca Dataset}
Based on the experimental setup of~\cite{cosmo2017consistent} using the Tosca dataset~\cite{bronstein2008numerical}, we compare our method with two other approaches that guarantee cycle-consistency, namely \textsc{QuickMatch}, and the sparse multi-shape matching approach by Cosmo~et~al.~\cite{cosmo2017consistent}. The feature descriptors on the shape surfaces are based on \emph{wave kernel signatures (WKS)} \cite{aubry2011wave}, we use \textsc{QuickMatch}$~(\rho_{\text{den}} {=} 0.2$) as initialisation, and set $\mu {=} {5}$. 

\paragraph{Multi-shape matching:} Quantitative results are shown in Fig.~\ref{fig:resultsTosca} and qualitative results are shown in Fig.~\ref{fig:qualitativeResultsTosca}. As explained before, \textsc{QuickMatch} ignores geometric relations between points, and thus leads to geometric inconsistencies (first row of Fig.~\ref{fig:qualitativeResultsTosca}).
While the method of Cosmo~et~al.~\cite{cosmo2017consistent} is able to incorporate geometric relations between points, one major limitation of their approach is that only a sparse subset of matchings is found. This may happen even when the shape collection is outlier free~\cite{cosmo2017consistent}. This behaviour can be seen in the second row of Fig.~\ref{fig:qualitativeResultsTosca}, where only few multi-matchings are obtained and hence there are large regions for which no correspondences are found. In contrast, our approach incorporates geometric consistency, produces significantly more multi-matchings (Fig.~\ref{fig:qualitativeResultsTosca}), and results in a percentage of correct keypoints (PCK) that is competitive to the method of Cosmo~et~al.~\cite{cosmo2017consistent}, see Fig.~\ref{fig:resultsTosca}. %

\section{Discussion \& Limitations}
The monotonicity properties of Alg.~\ref{alg:hppi}  rely on the assumption that the multi-adjacency matrix $A$ is positive semidefinite (Prop.~\ref{prop:mon}), which usually is not an issue in practice.
As explained in \cite{vestner2017efficient}, for getting rid of a bias towards far-away points, one usually applies a kernel (e.g.~Gaussian or heat kernel) to pairwise point distances, which directly results in positive semidefinite adjacency matrices~\cite{Bishop:2006ug}.

As the multi-matching problem
is non-convex, there is generally no guarantee to obtain the global  optimum. Moreover, the results produced by (most) methods
are dependent on the initialisation.
To analyse the sensitivity regarding the initialisation,
 we additionally evaluated our method on the HiPPI dataset with random initialisations, which achieved an average fscore of $0.90{\pm}0.08$ (\textsc{QuickMatch}: $0.71{\pm}0.03$; ours with \textsc{QuickMatch} init.: $0.92{\pm}0.05$, cf.~Fig.~\ref{fig:resultsJohanMim}). Hence, 
our method is not substantially affected by the initialisation and is able to achieve high-quality results even from random initialisations.

Our formulation of the geometric consistency term in Problem~\eqref{eq:mgm} corresponds to the Koopmans-Beckmann form of the QAP~\cite{Koopmans:1957gf}, which is strictly less general compared to Lawler's form $\VEC(X)^T W \VEC(X)$~\cite{Lawler:1963wn}. An interesting direction for future work is to devise an analogous algorithm for solving Lawler's form.

\section{Conclusion}
We presented a higher-order projected power iteration approach for multi-matching. Contrary to existing permutation synchronisation methods \cite{Pachauri:2013wx,Chen:2014uo,zhou2015multi,Shen:2016wx,Arrigoni:2017ut,Maset:YO8y6VRb}, our method takes geometric relations between points into account. %
Hence, our approach can be seen as a generalisation of permutation synchronisation. %
Moreover, previous multi-matching methods that consider geometric consistency~\cite{Yan:2015vc,Yan:2016vf} only allow to solve problems with up to few thousand points. In contrast, we demonstrated that our approach scales to tens of thousands of points. 
In addition to being able to account for geometric consistency, key properties of our method are convergence guarantees, efficiency, simplicity, and guaranteed cycle-consistency. Moreover, we demonstrated superior performance on three datasets, which highlights the practical relevance of our method. %

\section*{Acknowledgements}
This work was funded by the ERC Consolidator Grant 4DRepLy. We thank Thomas M\"{o}llenhoff for helpful input regarding the convergence, and Franziska M\"{u}ller for providing feedback on the manuscript.

{\small
\bibliographystyle{ieee}
\bibliography{myreferences}
}

\end{document}